\documentclass[sigconf]{acmart}

\AtBeginDocument{%
  }

\setcopyright{acmlicensed}
\copyrightyear{2018}
\acmYear{2018}
\acmDOI{XXXXXXX.XXXXXXX}

\acmConference[Conference acronym 'XX]{Make sure to enter the correct
  conference title from your rights confirmation emai}{June 03--05,
  2018}{Woodstock, NY}
\acmISBN{978-1-4503-XXXX-X/18/06}

\begin{document}

\title{Scale-Invariant Learning-to-Rank}


\author{Alessio Petrozziello}
\affiliation{%
  \institution{Expedia Group}
  \city{London}
  \country{United Kingdom}}

\author{Christian Sommeregger}
\affiliation{%
  \institution{Expedia Group}
  \city{London}
  \country{United Kingdom}}

\author{Ye-Sheen Lim}
\affiliation{%
  \institution{Expedia Group}
  \city{London}
  \country{United Kingdom}}

\renewcommand{\shortauthors}{Sommeregger and Petrozziello et al.}

\begin{abstract}
    At Expedia, learning-to-rank (LTR) models plays a key role on our website in sorting and presenting information more relevant to users, such as search filters, property rooms, amenities, and images. A major challenge in deploying these models is ensuring consistent feature scaling between training and production data, as discrepancies can lead to unreliable rankings when deployed. Normalization techniques like feature standardization and batch normalization could address these issues but are impractical in production due to latency impacts and the difficulty of distributed real-time inference. To address consistent feature scaling issue, we introduce a scale-invariant LTR framework which combines a deep and a wide neural network to mathematically guarantee scale-invariance in the model at both training and prediction time. We evaluate our framework in simulated real-world scenarios with injected feature scale issues by perturbing the test set at prediction time, and show that even with inconsistent train-test scaling, using framework achieves better performance than without.
\end{abstract}


\received{20 February 2007}
\received[revised]{12 March 2009}
\received[accepted]{5 June 2009}

\maketitle

\section{Background}

Many e-commerce companies such as Expedia employ A/B testing \cite{dixon2011b, cheng2016wide, haldar2019applying} to decide whether a new product or feature is good enough to be rolled out to all customers. In large organizations it is very likely that changes rolled out for testing in different teams would affect the data relied upon by a deployed learning-to-rank (LTR) model without the machine learning engineer even knowing about it \cite{sculley2015hidden, iqbal2019production}. For example, an A/B test might be performed for how the price of a room is displayed to the user: one group of users can be presented with nightly prices, while another group with the price of the full stay. An LTR model tracking room price as a feature will now see room prices at different scales, and may output differently ranked items for each group of users given the same query. Other examples could be a change in currency, a change of the guest rating range, or hotel-landmarks distances expressed in kilometers instead of miles. A recommender system should be scale-invariant to these features \cite{pennock2000social}.

One common solution to this issue is normalizing \cite{yuan2019scaling, ba2016layer} the training and test data, along with the data seen during real-time inference at production. However, real-time prediction time normalization in large-scale deployed applications are often infeasible due to increased latency \cite{jake2008user, bai2017understanding, arapakis2014impact}, distributed inference limits access to complete item information leading to unreliable normalization, and list-wise standardization is ineffective with truncated lists in real-world environments. 

\section{Scale-Invariant Learning-to-Rank}

Faced with this issue, we introduced a scale-invariant framework for our LTR models. In a LTR model, we generally have a scoring function $f(\mathbf{x}_{ij})$ that scores an item $j$ in a query $i$ to rank the item within the query. For scale-invariance, we propose the following scoring function:

\begin{equation}
    f(\mathbf{x}_{ij}) = f_d(\mathbf{x}^Q_{i},\mathbf{x}_{ij}^{I,F}) + f_w(\mathbf{x}^Q_{i},\mathbf{x}_{ij}^{I,S}) \\
\label{eq:fn}
\end{equation}

\noindent
where, $\mathbf{x_ij}$ are the features associated with item $j$ in list $i$, $\mathbf{x}^Q_{i}$ are query-wide features, while $\mathbf{x}_{ij}^{I,F}$ and $\mathbf{x}_{ij}^{I,S}$ are item features with and without scaling issues respectively. The scoring function above is essentially a neural network consisting of a deep path $f_d$ and a wide path $f_w$. $f_d$ is a feedforward neural network, and if $f_w$ is defined properly, it can be mathematically proven that any perturbation to $\mathbf{x}_{ij}^{I,S}$ can be canceled out when computing the difference in score for any two items $f(\mathbf{x}_{ij})-f(\mathbf{x}_{ik})$. Hence, the ranking of items is preserved for a list of items in a query. The provided supplementary materials can be referred to for the proofs. An example of $f_w$ we have implemented is with a cross-product as follows:

\begin{equation}
    f_w(\mathbf{x}^Q_{i},\mathbf{x}_{ij}^{I,S}) = <\mathbf{w}, (\mathbf{f}_s(\mathbf{x}^{Q}_{i}) \otimes_{kron} \log(\mathbf{x}^{I,S}_{ij}))>
\label{eq:fw}
\end{equation}

\noindent
where, $\mathbf{w}$ are weights that can be set to "pass-through". 

Since we only need the scoring function for scale-invariance, the overall framework is model-agnostic and can be integrated into the widely used industrial standard ListNet \cite{listnet} or ListMLE \cite{listmle}.

\section{Evaluation}

We evaluate the proposed scoring function on three LTR datasets. One is an internal dataset labelled as \emph{ExpediaHotels}, and the other two public datasets labelled as \emph{RecTour} and \emph{MSLR} respectively. \emph{ExpediaHotels} contains 56 million vacation property searches and bookings on Expedia Group's platform where the task is to rank the properties for a given search. The \emph{RecTour} dataset \cite{woznica2021present} is the same as our internal dataset but anonymised for public research. The \emph{MSLR} dataset \cite{qin2013introducing} is Microsoft's public Web30k dataset comprising around 30 thousand queries where the task is to rank web documents for a given search.

For each dataset, we evaluate the performance of ListNet and ListMLE, with and without our proposed scale-invariant (SIR), for unperturbed and perturbed cases. In the perturbed cases, we multiple a chosen feature in the test set with a scaling factor, causing scaling issues at prediction time. The feature and scaling factors are selected based on possible occurrence in practice.

For \emph{ExpediaHotels}, the property price feature uses the full price of the length of stay instead of nightly price, and change the price to the currency to the local currency of the property, instead of standardising to USD. This issues occur on a frequent basis as different A/B tests are performed.

For \emph{RecTour}, eventhough the dataset is  similar to \emph{ExpediaHotels}, the price feature is anonymised and we have to perturb another feature. The customer rating feature can often change based on different tests to present different rating ranges to users. In this dataset the customer rating feature is multiplied by a factor of 10 at prediction time.

In the case of the MSLR dataset, identifying features to perturb for scaling issues at prediction time proves to be a challenging task, as no features exhibit readily apparent scaling concerns. We propose to perturb the number of slash in the URL of a document in a query. Although it may sound like an arbitrary and irrelevant feature to select, it has been shown by \cite{han2018feature} that it is a highly relevant label in the dataset. URLs could also be subject to changes as metadata is injected into the link while a user browses a website, which could justify the possibility of the feature being scale variant at prediction time. Therefore, the feature tracking the number of slashes in a URL is multiplied by 100 here at prediction time. 

For the ExpediaHotels and the RecTour datasets, we randomly split the dataset into 70\% train and 30\% out-of-sample test sets for evaluation. For the MSLR dataset, 5 folds of train and test sets are provided. We evaluate across all 5 folds and average the performance. All trained models are equally tuned with classic hyper-parameter optimization. All models are trained on Databricks using Tensorflow 2.9 running on AWS g5.24xlarge clusters. Also, we use NDCG \cite{wang2013theoretical} as a measure of performance.

\begin{table}[t!]
\centering
\begin{tabular}{lll}
    \hline
    Model & Unperturbed & Perturbed \\ \hline
    ListNet & 0.526 & 0.518 \\
    ListNet (SIR) & 0.522 & 0.522 \\ \hline
    ListMLE & 0.668 & 0.639 \\
    ListMLE (SIR) & 0.667 & 0.667 \\
    \hline
\end{tabular}
\caption{The NDCG performance of ListNet and ListMLE, with and without SIR integration, on unperturbed and different perturbed test sets for the ExpediaHotels dataset.}
\label{table:ExpediaHotels}
\end{table}

\begin{table}[t!]
\centering
\begin{tabular}{lll}
    \hline
    Model & Unperturbed & Perturbed \\ \hline
    ListNet & 0.315 & 0.222 \\
    ListNet (SIR) & 0.339 & 0.324 \\ \hline
    ListMLE & 0.337 & 0.238 \\
    ListMLE (SIR) & 0.339 & 0.330 \\
    \hline
\end{tabular}
\caption{The NDCG performance of ListNet and ListMLE, with and without SIR integration, on unperturbed and different perturbed test sets for the RecTour dataset.}
\label{table:rectourdataset}
\end{table}

\begin{table}[t!]
\centering
\begin{tabular}{lll}
    \hline
    Model & Unperturbed & Perturbed \\ \hline
    ListNet & 0.681 & 0.603 \\
    ListNet (SIR) & 0.672 & 0.666 \\ \hline
    ListMLE & 0.690 & 0.655 \\
    ListMLE (SIR) & 0.701 & 0.692 \\
    \hline
\end{tabular}
\caption{The NDCG performance of ListNet and ListMLE, with and without SIR integration, on unperturbed and different perturbed test sets for the MSLR dataset.}
\label{table:mslrdataset}
\end{table}

In Tables \ref{table:ExpediaHotels}, \ref{table:rectourdataset} and \ref{table:mslrdataset}, first of all we can observe that for the unperturbed case, the performance between ListNet and ListMLE integrated with the proposed scale-invariance scoring function performs about the same. This passes our first test of do-no-harm to the existing models in which the proposed scoring function is integrated into. Then, we can observe that when the test set is perturbed, the models with SIR integrated in maintains its performance quite well when there are scaling issues in the test set that is not seen at training time.

\section{Conclusion}

We have proposed a scale-invariant LTR framework that can be integrated easily into any existing LTR approaches that accept an arbitrary scoring function. Through a realistic scenario in which we perturbed the test set at prediction time, we show that using our framework achieves better performance when there are scaling issues.

The current limitation of the proposed scale-invariant ranking function is that it only works for strictly positive numerical  features. A challenge remain in our future work includes extending this to work with perturbations that could be negative, as well as perturbations on categorical data. Another limitation is that the proposed ranking function does not detect scaling issue on its own. Therefore prior knowledge of which features would have scaling issues would be needed through domain knowledge and experience. If not known prior, all features would need to be assumed to have scaling issues and pass through the wide layer, potentially resulting in larger outer products that can be hard to optimize.

\bibliographystyle{ACM-Reference-Format}
\bibliography{ref}

\appendix

\section{Scale-Invariance Proof}

In learning-to-rank models, there is a ranking function that takes a list of items (represented by some vector of item-level features and query-level features) in a query, and scores each item to obtain an optimal sorted list of items. Before diving into the formal explanation of the proposed ranking function we establish the notations used throughout the paper in Table \ref{table_notation}.

\begin{table}[t!]
\centering
\begin{tabular}{lp{5cm}}
    $\mathbf{x}_{i}^Q \in \mathbb{R}^M$ & The query feature vector of query $i$. Items in the same query share the same query features. \\
    $\mathbf{x}_{ij}^{I} \in \mathbb{R}^{K_1+K_2}$ & The item feature vector of item $j$ in query $i$. It can be further split into $\mathbf{x}_{ij}^{I,F}$ and $\mathbf{x}_{ij}^{I,S}$. \\
    $\mathbf{x}_{ij}^{I,F} \in \mathbb{R}^{K_1}$ & The slice of item feature vector representing features whose scale remain the same in all cases. \\
    $\mathbf{x}_{ij}^{I,S} \in \mathbb{R}^{K_2}_{>0}$ & The partition of item features whose scale can vary at prediction time (to which our method is scale invariant). Note that the partition of features which have scaling issue is known. Here $\mathbb{R}_{>0}$ denotes positive real numbers. \\
    $\mathbf{x}_{ij} \in \mathbb{R}^{J}$ & The feature vector of the item $j$ in query $i$, where $J=M+K_1+K_2$. It is the concatenation of $\mathbf{x}_{i}^Q$ and item features $\mathbf{x}_{ij}^{I}$. \\
    $y_{ij}$ & Labeled score of item $j$ in query $i$ where $y_{ij}=1$ for the relevant item and $y_{ij}=0$ otherwise. \\
    $D_i$ & The number of items in a query $i$. \\
    $N$ & The total number of queries in the dataset.
\end{tabular}
\caption{Notations used in this paper.}
\label{table_notation}
\end{table}

The proposed SIR ranking function is defined as $f_{n}: \mathbb{R}^{M+K_1+K_2} \rightarrow \mathbb{R}$ is composed of two paths. One is the deep path defined as $f_d:\mathbb{R}^M \times \mathbb{R}^{K_1} \rightarrow \mathbb{R}$ which takes $\mathbf{x}_{i}^{Q}$ and $\mathbf{x}^{I,F}_{ij}$ as input and calculates the score of the items using a feed-forward neural network and a dense layer containing the scores. The other is the wide path defined as $f_w:\mathbb{R}^M \times \mathbb{R}^{K_1+K_2} \rightarrow \mathbb{R}$ which takes $\mathbf{x}_{i}^{Q}$ and $\mathbf{x}^{I,S}_{ij}$ as input, and linearly interacts them through their outer product. The output of both path are summed and passed to a soft-max function to convert the sumFmed outputs into probabilities, which can then be interpreted as scores. Note that the Siamese network form is for visualisation only. In practice, it is much more computationally efficient to represent the sub-networks as an additional dimension of a tensor.

More specifically, the score of item $j$ in query $i$ can be computed using the scoring function $f_n$ as follows:

\begin{equation}
    f_{n}(\mathbf{x}_{ij}) = f_d(\mathbf{x}^Q_{i},\mathbf{x}_{ij}^{I,F}) + f_w(\mathbf{x}^Q_{i},\mathbf{x}_{ij}^{I,S}) \\
\label{eq:fn1}
\end{equation}
\begin{equation}
    f_w(\mathbf{x}^Q_{i},\mathbf{x}_{ij}^{I}) = <\mathbf{w}, (\mathbf{f}_s(\mathbf{x}^{Q}_{i}) \otimes_{kron} \log(\mathbf{x}^{I,S}_{ij}))>
\label{eq:fw1}
\end{equation}

\noindent
where, $\mathbf{f}_s: \mathbb{R}^{M} \rightarrow \mathbb{R}^L$. Here $\mathbf{f}_{s}$ compresses the query features by projecting them to $\mathbb{R}^L$ with $L<M$, $\otimes_{kron}$ is the Kronecker product, $\mathbf{w}$ is a weight parameter to tune, and $<\cdot>$ represents the inner product operator. In the rest of this section, we proof that the ranking function $f_n$ theoretically guarantees scale-invariance in features $\mathbf{x}_{ij}^{I,S}$. 

Given two items $\mathbf{x}_{ij}$ and $\mathbf{x}_{ik}$ in the same query $i$, the score difference between $\mathbf{x}_{ij}$ and $\mathbf{x}_{ik}$ (and hence the ordering between the two items) is given by:

\begin{align}
\label{eq:proof0}
    f_n(&\mathbf{x}_{ij}) - f_n(\mathbf{x}_{ik}) \notag \\
    = & \left( f_d(\mathbf{x}^Q_{i},\mathbf{x}_{ij}^{I,F}) + f_w(\mathbf{x}^Q_{i},\mathbf{x}_{ij}^{I}) \right) \notag \\
    & - \left( f_d(\mathbf{x}^Q_{i},\mathbf{x}_{ik}^{I,F}) + f_w(\mathbf{x}^Q_{i},\mathbf{x}_{ik}^{I}) \right)
\end{align}

Now assume the scale of $\mathbf{x}_{i}^{I,S}$ increases by $c>0$ times, i.e. $\tilde{\mathbf{x}}_{ij}^{I,S} = c\mathbf{x}_{ij}^{I,S}$ and $\tilde{\mathbf{x}}_{ik}^{I,S} = c\mathbf{x}_{ik}^{I,S}$ where $c$ is a scaling factor which can take any positive value. Then the score difference of these two items after the scale change is given by:

\begin{align}
\label{eq:proof1}
    f_n(&\tilde{\mathbf{x}}_{ij}) - f_n(\tilde{\mathbf{x}}_{ik}) \notag \\
    = & \left( f_d(\mathbf{x}^Q_{i},\mathbf{x}_{ij}^{I,F}) + f_w(\mathbf{x}^Q_{i},\tilde{\mathbf{x}}_{ij}^{I}) \right) \notag \\
    & - \left( f_d(\mathbf{x}^Q_{i},\mathbf{x}_{ik}^{I,F}) + f_w(\mathbf{x}^Q_{i},\tilde{\mathbf{x}}_{ik}^{I}) \right)
\end{align}

\noindent
Now we would like to prove the following:

\begin{theorem}
The score difference between $\mathbf{x}_{ij}$ and $\mathbf{x}_{ik}$ does not change before and after scaling, i.e. $f_n(\tilde{\mathbf{x}}_{ij}) - f_n(\tilde{\mathbf{x}}_{ik})=f_n(\mathbf{x}_{ij}) - f_n(\mathbf{x}_{ik})$.
\end{theorem}

\begin{proof}

\begin{align}
\label{eq:proof2}
    f_n(&\tilde{\mathbf{x}}_{ij}) - f_n(\tilde{\mathbf{x}}_{ik}) \notag \\
    = & \left( f_d(\mathbf{x}^Q_{i},\mathbf{x}_{ij}^{I,F}) + f_w(\mathbf{x}^Q_{i},\tilde{\mathbf{x}}_{ij}^{I,S}) \right) \notag \\
    & - \left( f_d(\mathbf{x}^Q_{i},\mathbf{x}_{ik}^{I,F}) + f_w(\mathbf{x}^Q_{i},\tilde{\mathbf{x}}_{ik}^{I,S}) \right) \notag \\
    = & f_d(\mathbf{x}^Q_{i},\mathbf{x}_{ij}^{I,F}) + <\mathbf{w},
    \left( \mathbf{f}_s(\mathbf{x}^{Q}_{i}) \otimes_{kron} \log(\tilde{\mathbf{x}}^{I,S}_{ij}) \right) > \notag \\
    & - f_d(\mathbf{x}^Q_{i},\mathbf{x}_{ik}^{I,F}) - <\mathbf{w}, \left( \mathbf{f}_s(\mathbf{x}^{Q}_{i}) \otimes_{kron} \log(\tilde{\mathbf{x}}^{I,S}_{ik}) \right)> \notag \\
    = & f_d(\mathbf{x}^Q_{i},\mathbf{x}_{ij}^{I,F}) + <\mathbf{w}, \left( \mathbf{f}_s(\mathbf{x}^{Q}_{i}) \otimes_{kron} \log(c\mathbf{x}^{I,S}_{ij}) \right) > \notag \\
    & - f_d(\mathbf{x}^Q_{i},\mathbf{x}_{ik}^{I,F}) - <\mathbf{w},\left (\mathbf{f}_s(\mathbf{x}^{Q}_{i}) \otimes_{kron} \log(c\mathbf{x}^{I,S}_{ik}) \right) > \notag \\
    = & f_d(\mathbf{x}^Q_{i},\mathbf{x}_{ij}^{I,F}) + <\mathbf{w}, \Big( \mathbf{f}_s(\mathbf{x}^{Q}_{i}) \otimes_{kron} \log(\mathbf{x}^{I,S}_{ij}) \Big)> \notag \\ 
    & + <\mathbf{w}, \Big( \mathbf{f}_s(\mathbf{x}^{Q}_{i}) \otimes_{kron} \log(c) \Big) > \notag \\
    & - f_d(\mathbf{x}^Q_{i},\mathbf{x}_{ik}^{I,F}) - <\mathbf{w}, \Big( \mathbf{f}_s(\mathbf{x}^{Q}_{i}) \otimes_{kron} \log(\mathbf{x}^{I,S}_{ik})\Big )> \notag \\
    & - <\mathbf{w}, \Big( \mathbf{f}_s(\mathbf{x}^{Q}_{i}) \otimes_{kron} \log(c) \Big) > \notag \\
    = & f_d(\mathbf{x}^Q_{i},\mathbf{x}_{ij}^{I,F}) + <\mathbf{w}, (\mathbf{f}_s(\mathbf{x}^{Q}_{i}) \otimes_{kron} \log(\mathbf{x}^{I,S}_{ij}))> \notag \\ 
    & - f_d(\mathbf{x}^Q_{i},\mathbf{x}_{ik}^{I,F})- <\mathbf{w}, (\mathbf{f}_s(\mathbf{x}^{Q}_{i}) \otimes_{kron} \log(\mathbf{x}^{I,S}_{ik}))> \notag \\
    = & f_n(\mathbf{x}_{ij}) - f_n(\mathbf{x}_{ik}).
\end{align}

\end{proof}

Therefore, the proposed ranking function gives the same score difference for two items in the same search with or without scaling. The difference of scores is invariant to the scale change. Also, observe that the scale-invariance property is independent of the loss function of the network, meaning we can apply any state-of-the-art loss when training the network.

Accordingly, the ranking of a list of items based on this ranking function is scale invariant as well. In query $i$, assume we have scores $f_n(\mathbf{x}_{i1})$, $f_n(\mathbf{x}_{i2})$, $\dots$, $f_n(\mathbf{x}_{iD_i})$. Then a score function for query $i$, denoted $\pi_i$, is a bijective function from  $\{1, 2, \dots, D_i\}$ to $\{1, 2, \dots , D_i\}$:

\begin{equation}
\pi_i:\{1, 2, \dots, D_i\} \rightarrowtail\!\!\!\!\!\rightarrow \{1, 2, \dots , D_i\}
\end{equation}

\noindent
where, $\rightarrowtail\!\!\!\!\!\rightarrow$ indicates bijection. Here the items are ranked according to scores in descending order, i.e. our ranking function $\pi$ satisfies $f_n(\mathbf{x}_{i\pi_i^{-1}(1)}) >= f_n(\mathbf{x}_{i \pi_i^{-1}(2)})>=\ldots>=f_n(\mathbf{x}_{i\pi_i^{-1}(D_i)})$. The notation $\pi_i^{-1}(j)$ indicates the index of the item ranked in the $j$-th place in the query by the ranking function $\pi_i$.

\end{document}